\documentclass[nohyperref]{article}

\usepackage{microtype}
\usepackage{graphicx}
\usepackage{subfigure}
\usepackage{booktabs} 
\usepackage[utf8]{inputenc}
\usepackage[T1]{fontenc}
\usepackage{graphicx}
\usepackage{url}            
\usepackage{amsfonts}       
\usepackage{nicefrac}       
\usepackage{lipsum}
\usepackage{float}
\usepackage{array}
\usepackage{xspace}
\usepackage{multirow}
\usepackage[export]{adjustbox}
\usepackage{stfloats}
\usepackage{natbib}
\usepackage{xcolor}
\usepackage{color}
\usepackage{wrapfig}
\usepackage{bbm}
\usepackage{caption}
\usepackage{bm}

\usepackage{hyperref}       
\hypersetup{
    unicode=false,          
    pdftoolbar=true,        
    pdfmenubar=true,        
    pdffitwindow=false,      
    pdfnewwindow=true,      
    colorlinks=true,       
    linkcolor=DarkBlue,          
    citecolor=blue,        
    filecolor=DarkRed,      
    urlcolor=DarkBlue,          
}

\usepackage{algorithm,algorithmic}
\usepackage[algo2e]{algorithm2e}

 \usepackage[accepted]{icml2022}

\usepackage{amsmath}
\usepackage{amssymb}
\usepackage{mathtools}
\usepackage{amsthm}
\usepackage{thm-restate}

\usepackage[capitalize,noabbrev]{cleveref}

\theoremstyle{plain}
\newtheorem{theorem}{Theorem}[section]

\newtheorem{claim}[theorem]{Claim}
\theoremstyle{definition}
\newtheorem{definition}[theorem]{Definition}

\newtheorem{example}[theorem]{Example}
\theoremstyle{remark}
\newtheorem{remark}[theorem]{Remark}

\usepackage[textsize=tiny]{todonotes}

\DeclareCaptionLabelFormat{andtable}{#1~#2  \&  \tablename~\thetable}

\definecolor{DarkGreen}{rgb}{0.1,0.5,0.1}
\definecolor{DarkRed}{rgb}{0.5,0.1,0.1}
\definecolor{DarkBlue}{rgb}{0.1,0.1,0.5}
\definecolor{Gray}{rgb}{0.2,0.2,0.2}

\definecolor{negvals}{rgb}{0.1,0.5,0.1}
\definecolor{posvals}{rgb}{0.5,0.1,0.1}
\definecolor{neutralvals}{RGB}{255, 140, 0}
\definecolor{neutralvals}{rgb}{0.6, 0.51, 0.48}
\definecolor{posvals}{rgb}{0.8, 0.0, 0.0}

\DeclareMathOperator*{\E}{\mathbb{E}}
\DeclareMathOperator*{\prob}{\mathbb{P}}
\DeclareMathOperator*{\sgn}{\mathrm{sgn}}

\newcommand{\Ind}{\mathbb I}

\newcommand{\defeq}{\stackrel{\small \mathrm{def}}{=}}

\DeclareMathOperator{\Var}{var}

\DeclareMathOperator*{\argmax}{argmax}

\newcommand{\cA}{{\cal A}}

\newcommand{\cL}{{\cal L}}

\newcommand{\cO}{{\cal O}}
\newcommand{\cP}{{\cal P}}

\newcommand{\cS}{{\cal S}}

\newcommand{\cV}{{\cal V}}

\newcommand{\cX}{{\cal X}}
\newcommand{\cY}{{\cal Y}}

\newcommand{\vx}{\mathbf x}
\newcommand{\vr}{\mathbf r}
\newcommand{\vy}{\mathbf y}

\newcommand{\vw}{\mathbf w}

\newcommand{\vv}{\mathbf v}
\newcommand{\vwhat}{\widehat{\mathbf w}}

\newcommand{\R}{\mathbb{R}}
\renewcommand \vec [1]{\mathbf{#1}}

\usepackage{enumitem}
\newcommand{\expalg}{\texttt{Exp4.VC} }
\newcommand{\vxi}{\boldsymbol{\xi}}

\newcounter{protocol}
\makeatletter
\newenvironment{protocol}[1]{%
  \let\c@algorithm\c@protocol
  \renewcommand{\ALG@name}{Model}
  \begin{algorithm}#1%
  }{\end{algorithm}
}
\makeatother

\icmltitlerunning{No-Regret Learning in Partially-Informed Auctions}

\begin{document}

\twocolumn[
\icmltitle{No-Regret Learning in Partially-Informed Auctions}




\begin{icmlauthorlist}
\icmlauthor{Wenshuo Guo}{yyy}
\icmlauthor{Michael I. Jordan}{yyy,comp}
\icmlauthor{Ellen Vitercik}{sch}
\end{icmlauthorlist}

\icmlaffiliation{yyy}{Department of Electrical Engineering \& Computer Sciences, University of California, Berkeley, USA}
\icmlaffiliation{comp}{Department of Statistics, University of California, Berkeley, USA}
\icmlaffiliation{sch}{Department of Management Science \& Engineering and Department of Computer Science, Stanford University, USA}

\icmlcorrespondingauthor{Wenshuo Guo}{wguo@cs.berkeley.edu}

\icmlkeywords{Machine Learning, ICML}

\vskip 0.3in
]



\printAffiliationsAndNotice{}  

\begin{abstract}
Auctions with partially-revealed information about items are broadly employed in real-world applications, but the underlying mechanisms have limited theoretical support. In this work, we study a machine learning formulation of these types of mechanisms, presenting algorithms that are no-regret from the buyer's perspective. Specifically, a buyer who wishes to maximize his utility interacts repeatedly with a platform over a series of $T$ rounds. In each round, a new item is drawn from an unknown distribution and the platform publishes a price together with incomplete, ``masked'' information about the item. The buyer then decides whether to purchase the item. We formalize this problem as an online learning task where the goal is to have low regret with respect to a myopic oracle that has perfect knowledge of the distribution over items and the seller's masking function. When the distribution over items is known to the buyer and the mask is a SimHash function mapping $\R^d$ to $\{0,1\}^{\ell}$, our algorithm has regret $\tilde \cO((Td\ell)^{\nicefrac{1}{2}})$. In a fully agnostic setting when the mask is an arbitrary function mapping to a set of size $n$ and the prices are stochastic, our algorithm has regret $\tilde \cO((Tn)^{\nicefrac{1}{2}})$.

\end{abstract}

\section{Introduction}\label{sec:intro}

Selling mechanisms play a crucial role in economic theory and have a wide range of applications across many industries~\citep{post1995application,milgrom2004putting, edelman2007internet, milgrom2010simplified, arnosti2016adverse}. Under the canonical mechanism design model, buyers choose whether or not to buy items for sale based on their true values for those items. This fundamental model, however, assumes that the buyers know exactly how much they value the items for sale, which is often not the case.

One of the overriding reasons that a buyer may not know their true values is information asymmetry: the seller may purposefully obfuscate information about an item for sale. For example, the seller may hide information about the item in the hopes of better revenue~\citep{gershkov2009optimal}. Alternatively, information about the item may be private, and thus the seller may wish to protect this sensitive information by only revealing partial information about the item. For instance, in online advertising auctions, bids represent how much advertisers are willing to pay to display their ad to a particular user. Historically, advertisers have bid based on uniquely identifying information about users, but there has been a growing effort to protect users' privacy by obfuscating this sensitive information~\citep{juels2001targeted, guha2011privad, epasto2021clustering}.

In these scenarios, the buyer only has partial information about the item for sale but still must decide whether to make a purchase. This raises the question: \textbf{how should a buyer determine their purchase strategy with only incomplete item information?}

We study posted-price auctions---a fundamental mechanism family that is appealingly interpretable---with incomplete item information. In particular, the seller reveals obfuscated (``masked'') information about the item using a fixed, unknown masking function.
We study an online setting where, at each round, a fresh item is drawn from an unknown distribution (for example, a distribution over users visiting a webpage). The seller sets a price and the buyer chooses whether to buy the item based on the incomplete information that the seller provides. We propose no-regret learning algorithms for the buyer that achieve sub-linear regret compared to an oracle buyer who has perfect knowledge of the item distribution as well as the seller's masking function.

	%
	%
	%
	%
	%
	%

\subsection{Our results}

We study no-regret learning with incomplete item information in two settings:

\begin{enumerate}[itemsep=0mm]
    \item First, we propose an algorithm for a setting where the item distribution is known to the buyer and the mask is a SimHash function mapping $[0,1]^d$ to $\{0,1\}^\ell$.
    In other words, each item is defined by $d$ real-valued features and the seller reveals $\ell$ bits about the item to the buyer, as defined by a function that is unknown to the buyer. 
    This model has been studied from an applied perspective in the context of ad auctions~\citep{epasto2021clustering}. We provide an algorithm with regret $\cO(\sqrt{Td\ell\log (\nicefrac{T\ell}{\delta})})$.
    \item Next, we study a setting where the masking function is an arbitrary mapping from the set of all items, denoted $\cX$, to a finite set of size $n$. We propose an online learning algorithm with regret $\cO(\sqrt{T(n\log \nicefrac{T}{n} + \log \nicefrac{1}{\delta})})$ when the prices are stochastic, where $T$ is the length of the horizon. 
\end{enumerate}

In the first setting where the masking function is a SimHash function mapping $[0,1]^d$ to $\{0,1\}^{\ell}$, the domain of the masking function is of size $n = 2^{\ell}$, so our regret bound of $\tilde \cO((Td\ell)^{\nicefrac{1}{2}})$ is exponentially better than the latter regret bound. 

We summarize these results in Table~\ref{table:result}.

\begin{table*}
	\centering
	\renewcommand{\arraystretch}{1.2}
	\begin{tabular}{lllll}
		\toprule
		\textbf{Item distribution} & \textbf{Prices} & \textbf{Masking function $h$} & \textbf{Regret} \\\midrule
		Known & Adversarial & SimHash $h: [0,1]^d \to \{0,1\}^{\ell}$ & $\cO(\sqrt{Td\ell\log (\nicefrac{T\ell}{\delta})})$ (Theorem~\ref{thm:popu-regret})\\\midrule
		Unknown & Stochastic & Arbitrary $h : \cX \to [n]$ & $\cO(\sqrt{T(n\log \nicefrac{T}{n} + \log \nicefrac{1}{\delta})})$ (Theorem~\ref{thm:exp4}) \\\midrule
		Unknown & Adversarial & Arbitrary $h: \cX \to [n]$ & $\tilde \cO(T^{\nicefrac{2}{3}}n^{\nicefrac{1}{3}})$ (Remark~\ref{rem:sliv})\\\bottomrule
	\end{tabular}
	\caption{Summary of regret bounds which hold with probability at least $1-\delta$.}
	\label{table:result} \vspace{-2mm}
\end{table*}

\subsection{Related work}\label{sec:related_works}

This work draws on several threads of research on designing auctions with incomplete information, learning to bid, and privacy-preserving simple auctions.

\paragraph{Auction design with incomplete information.} Auctions with incomplete value information have attracted much research attention. Several prior works have explored auction design where the information about the item may be incomplete to the buyer or the seller~\citep{ganuza2004ignorance, esHo2007optimal, bergemann2007information, arefeva2021revealing, roesler2017buyer, li2017discriminatory, bergemann2021selling, li2017cheap}. In particular, ~\citet{ganuza2004ignorance} studied the incentives of the auctioneer to release signals about the item to the buyers that refine their private valuations before a second-price auction. In their model, the seller reveals a noisy item feature vector which is an unbiased estimator of the true one. \citet{bergemann2007information} considered single-item multi-bidder auctions where the seller decides how accurately the bidders can learn their valuations. However, all of this prior work is limited to offline settings; they did not explore online purchase strategies that the buyer can adopt. To the best of our knowledge, this work constitutes the first analysis of no-regret learning algorithms in partially-informed posted-price auctions.

\paragraph{Learning to bid with
an unknown value.} Instead of focusing the problem from the side of the auctioneer who aims to maximize revenue over repeated rounds, another active line of research studies bidding strategies for the bidders when they do not know their values~\citep{dikkala2013can, weed2016online, feng2018learning, balseiro2019learning}.~\citet{feng2018learning} considered a single-item multi-bidder setting where the bidder learns to bid via partial feedback, and provide algorithms with regret rates against the best fixed bid in hindsight. \citet{dikkala2013can} explored a setting where bidders need to experiment in order to learn their valuations. The key difference between this work and ours is that our algorithms exploit the available ``partial'' information instead of having zero knowledge about the item being sold. This partial information model allows us to trade off between the amount of information revealed about the item and the regret. Moreover, we compete with the best purchase policy, rather than the best fixed bid in hindsight.

\paragraph{Private auctions.} Our theoretical model is motivated by recent work on designing privacy-enhanced auctions for practical usage. An important application of these auctions is online advertising, where the items being auctioned are user queries and the auctioner must trade off between user privacy and revenue maximization~\citep{de2016online,rafieian2021targeting, epasto2021clustering, guha2011privad, juels2001targeted}. \citet{epasto2021clustering} present a detailed exploration of Chrome's Federated Learning of Cohorts (FLoC) API, where user information is masked. Our contribution is to provide a formal treatment of such auctions, including providing algorithms that have theoretical guarantees in the setting of arbitrary masking functions that are unknown to the buyer.

\section{Preliminaries}\label{sec:prelim}

We begin by defining formally the setting we study for auctions with partial information.

\subsection{Setup}

We consider a setting where there is a single seller and a single buyer. There is a distribution $\cP$ over items which are elements of an abstract set $\cX$. The buyer has a bounded valuation, $v^*(\vx) \in [0,H]$ for every item $\vx \in \cX$ and some $H \in \R_{+}$\footnote{The assumption that there is a bound on the maximum amount that the agents value the item is widely made in prior research. In our main application---ad auctions---the value of an impression is typically very cheap, so this assumption is mild.}. The seller and buyer interact over a series of $T$ rounds. At each round $t \in [T]$, the seller draws an item $\vx_t \stackrel{iid}{\sim} \cP$ and sets a price $p_t(\vx_t) \in [0,H]$ for the item. The seller does not reveal $\vx_t$ to the buyer, but rather reveals some partial information $h(\vx_t) \in \cY$ where $h : \cX \to \cY$ is a fixed ``masking'' function that maps to an abstract set $\cY$.

We assume that the price function $p_t$ does not give any more information about $\vx_t$ than that is provided by the masking function $h$. In other words, if $h(\vx) = h(\vx')$, then $p_t(\vx) = p_t(\vx')$, which implies \begin{equation}\E[v^*(\vx) \mid h(\vx), p(\vx)] = \E[v^*(\vx) \mid h(\vx)].\label{eq:p}\end{equation}

The buyer uses a strategy $s_t : \cY \times \R \to \{0,1\}$ to decide whether to buy the item given the partially revealed item information and the price. Here $s_t(h(\vx_t), p_t(\vx_t)) = 1$ if and only if the buyer buys the item. Letting $b_t = s_t(h(\vx_t), p_t(\vx_t))$, the buyer's utility is $u_t = b_t(v^*(\vx_t) - p_t(\vx_t)) \in [-H, H]$. We summarize 
this process below.

\begin{protocol}
\caption{Online model}
\label{alg:online}
\begin{algorithmic}[1]
\FOR{$t=1,2,\ldots, T$}
\STATE Seller selects a price function $p_t: \cX \to [0,H]$.
\STATE Item $\vx_t$ is sampled from $\cP$.
\STATE Seller publishes item information $h(\vx_t)$ and a price $p_t(\vx_t)$.
\STATE Buyer decides whether or not to buy: $b_t = s_t(h(\vx_t), p_t(\vx_t)) \in \{0,1\}$.
\STATE Buyer obtains reward $u_t = (v^\ast(\vx_t) - p_t(\vx_t))\cdot b_t$.
\IF[\textcolor{blue}{item is purchased}]{$b_t$}
\STATE Buyer observes $\vx_t$.
\ENDIF\\[1ex]
\ENDFOR\\[1ex]
\end{algorithmic}
\end{protocol}



\begin{example}[Advertising auctions]
We instantiate our model in the context of advertising auctions, taking inspiration from \citet{epasto2021clustering}. The seller is a platform and the buyer is an advertiser. Each item $\vx \in \cX$ describes a user who visits the platform. For example, $\cX = \R^d$ might denote features that uniquely identify each user. On round $t$, the advertiser has a value $v^*(\vx_t)$ for the opportunity to show the user $\vx_t$ an ad. In order to preserve user privacy, the platform does not reveal $\vx_t$ to the advertiser, but rather some summary $h(\vx_t)$. For example, \citet{epasto2021clustering} study a setting where $h$ is a SimHash function, so $\vx_t \in \R^d$ and $h(\vx_t) \in \R^\ell$ for some $\ell < d$. The platform sets a price $p_t(\vx_t)$ which the advertiser pays to show the user an ad.
\end{example}

We study this problem from the perspective of the buyer: how should they select the strategy $s_t$ at each round to maximize their utility?
We study two settings: a setting where the distribution $\cP$ is unknown to the buyer and the prices are stochastic (Section~\ref{sec:finite-adaptive}) and a model where $\cP$ is known to the buyer with adversarial prices (Section~\ref{sec:population}).

\subsection{Regret and the optimal strategy}

We measure the regret of the buyer in our online model with regard to the optimal strategy $s^*$ of a myopic buyer\footnote{A myopic buyer optimizes his utility separately in each round.} who has perfect knowledge of the distribution $\cP$ and the masking function $h$, but not the realized item $\vx_t$. To make this dependence on the environment clear, in any single round, we use the notation $s^*(h(\vx), p(\vx), h, \cP)$ to denote the optimal strategy (we drop the subscript $t$ for simplicity). More formally, $s^*$ maximizes the expected utility: 
\[\argmax_{s \in \cS}\E_{\vx \sim \cP}[(v^\ast(\vx) - p(\vx))s(h(\vx), p(\vx), h, \cP) \mid h(\vx)],\]
where $\cS$ represents the set of all decision functions $s(\cdot): \cY \times \R \times \{h(\cdot), \cP\} \rightarrow \{0,1\}$.

\begin{definition}\label{def:regret}
The buyer's (expected) regret with respect to the optimal strategy $s^*$, denoted $R_T$, is defined as
\begin{align}\label{eq:regret}
&\E\bigg[\sum_{t=1}^T \bigg(v^\ast(\vx_t) - p_t(\vx_t)\bigg)s^*\left(h(\vx_t), p_t(\vx_t), h(\cdot), \cP\right)\nonumber\\ 
&- \bigg(v^\ast(\vx_t) - p_t(\vx_t)\bigg)s_t\left(h(\vx_t), p_t(\vx_t)\right)\bigg].
\end{align}
\end{definition}

In the following proposition, we identify the form of the optimal strategy $s^*$. The proof is in Appendix~\ref{app:prelim-proofs}.

\begin{restatable}{proposition}{opt}\label{prop:opt}
The strategy $s^*$ that maximizes $\E_{\vx \sim \cP}[(v^\ast(\vx) - p(\vx))s^*(h(\vx), p(\vx), h, \cP) \mid h(\vx)]$ is \[s^\ast(h(\vx), p(\vx), h, \cP) = \Ind\left(\E_{x \sim \cP}\left[v^*(
\vx) \mid h(\vx)\right] > p(\vx)\right).\]
\end{restatable}

Even when the buyer has no information about  the distribution $\cP$ (Section~\ref{sec:finite-adaptive}), we show that he can guarantee low regret with respect to $s^*$ with either stochastic or adversarial prices, in polynomial per-round runtime. When the distribution $\cP$ is known (Section~\ref{sec:population}), we provide an algorithm with exponentially better regret. 



\section{Known Item Distribution}\label{sec:population}


First, we focus on a specific class of masking functions, SimHash, motivated by recent practical applications in ad auctions~\citep{epasto2021clustering}. Here, $\cX$ is a feature space $[0,1]^d$ and the masking function $h$ is a SimHash function that is unknown to the buyer. In other words, there are $\ell$ unknown vectors $\vec{w}_1, \dots, \vec{w}_\ell \in \R^d$ such that the masking function, denoted as $h_{\vw}$ with $\vw = (\vw_1, \dots, \vw_\ell)$, is $h_{\vw}(\vec{x}) = (\sgn(\vec{w}_1\cdot \vec{x}), \dots, \sgn(\vec{w}_\ell\cdot \vec{x}))^\top$.

We consider the setting where the distribution of the items is known to the buyer (for example, via historical data). We provide an algorithm that achieves a regret $\tilde\cO(\sqrt{Td\ell})$ even under adversarial prices. Since the masking function maps to a set of size $n = 2^{\ell}$, the regret only depends logarithmically on $n$. As we detail in the subsequent Section~\ref{sec:finite-adaptive}, this algorithm achieves exponentially better regret compared to the algorithm
we present where the distribution $\cP$ over items is unknown and the masking function is arbitrary.



\begin{algorithm}[!ht]
	\caption{Explore-then-Commit (Known Distribution)}\label{alg:explore-then-commit}
	\begin{algorithmic}[1]
		\STATE \textbf{Input: horizon $T$, distribution $\cP$, $d, \ell \in \mathbb{N_{+}}$, and $\delta \in (0,1)$.} 
		\STATE Compute $t' = \sqrt{4Td\ell\log(\nicefrac{\ell}{\delta})}$.
		\FOR[\textcolor{blue}{Exploration phase}]{$t=1,2,\ldots, t'$}
		\STATE Receive $h(\vx_t)$, price $p_t(\vx_t)$ where $\vx_t \stackrel{iid}{\sim} \cP$.
		\STATE Make decision $b_t = 1$ and observe $\vx_t$.
		\ENDFOR\\[1ex]
		
		\STATE Use linear programming to compute $\hat \vw = (\hat \vw_1, \dots, \hat \vw_\ell)$ such that $h_{\hat\vw}(\vx_i) = h_{\vw}(\vx_i)$ for all $i \in [t']$.\\[1ex]
		\FOR[\textcolor{blue}{Exploitation phase}]{$t=t'+1, t'+2 ,\ldots, T$}
		\STATE Receive $h_{\vw}(\vx_t)$, price $p_t(\vx_t)$ where $\vx_t \stackrel{iid}{\sim} \cP$.
		\STATE Obtain an estimate $\hat Z_t$ of $\E_{\vx \sim \cP}[v^\ast (\vx) \mid \vx \in h_{\vwhat}^{-1}(h_{\vw}(\vx_t))]$ using the Integration Algorithm by \citet{lovasz2006fast}. \label{alg:step-sampling}
		\STATE Make decision $b_t = \Ind(\hat Z_t \geq p_t(\vx_t))$.
		\ENDFOR\\[1ex]
	\end{algorithmic}
\end{algorithm}
Algorithm~\ref{alg:explore-then-commit} begins with an exploration phase of length $t' = \tilde{O}(\sqrt{Td\ell})$, during which the buyer buys the item in each round. The algorithm then uses linear programming to solve for separators $\hat{\vec{w}} = (\hat \vw_1, \dots, \hat \vw_\ell)$, $\hat \vw_j \in \R^d$ for all $j \in [\ell]$, such that $\sgn(\vw_j \cdot \vx_i) = \sgn(\hat \vw_j \cdot \vx_i)$ for all $j \in [\ell]$ and $i \in [t'].$ During the rest of the rounds $t \in \{t'+1, \dots, T\}$, the algorithm exploits. Since the optimal strategy is to buy if $\E_{\vx \sim \cP}\left[v^*(\vx) \mid h_{\vw}(\vx) = h_{\vw}(\vx_t)\right] \geq p(\vx_t)$ (Prop. \ref{prop:opt}), the algorithm uses $h_{\vw}(\vx_t)$ and $\hat\vw = (\hat \vw_1, \dots, \hat \vw_\ell)$ to compute an estimate of $\E_{\vx \sim \cP}\left[v^*(\vx) \mid h_{\vw}(\vx) = h_{\vw}(\vx_t)\right]$. 

The intuition behind the estimate is the following. Letting $h_{\vw}^{-1}(\vx_t) = \{\vec{x}:  h_{\vw}(\vx) = h_{\vw}(\vx_t)\}$ (a convex polytope), we have that $\E_{\vx \sim \cP}\left[v^*(\vx) \mid h_{\vw}(\vx) = h_{\vw}(\vx_t)\right] = \E_{\vx \sim \cP}\left[v^*(\vx) \mid \vx \in h_{\vw}^{-1}(h_{\vw}(\vx_t))\right]$. Since the buyer does not know $\vw$, we cannot compute the set $h_{\vw}^{-1}(h_{\vw}(\vx_t))$, but we can compute the set $h_{\hat \vw}^{-1}(h_{\vw}(\vx_t))$ using the estimated separators that we have obtained after the exploration phase. Even still, the conditional expectation $\E_{\vx \sim \cP}\left[v^*(\vx) \mid \vx \in h_{\hat \vw}^{-1}(h_{\vw}(\vx_t))\right]$ may be challenging to compute in high dimensions. Therefore, we use a sampling algorithm by \citet{lovasz2006fast} to compute an estimate $Z_t$ of $\E_{\vx \sim \cP}[v^\ast (\vx) \mid \vx \in h_{\vwhat}^{-1}(h_{\vw}(\vx_t))]$. The buyer buys the item if $Z_t \geq p_t(\vx_t)$.

To compute the estimate $Z_t$, \citet{lovasz2006fast} require that if $\pi$ is the density function of $\cP$, then $v^*(\vx)\pi(\vx)$ is log-concave and ``well-rounded.'' Many well-studied distributions are log-concave, including the normal, exponential, uniform, and beta distributions, among many others. Moreover, every concave function that is nonnegative on its domain is log-concave. If $v^*$ and $\pi$ are log-concave, then $v^*(\vx)\pi(\vx)$ is also log-concave. For example, the Cartesian product of single-dimensional log-concave distributions (exponential, logistic, extreme value, Laplace, and beta distributions, among many others) is log-concave. Log-concavity has also been widely-assumed in prior works in machine learning and high-dimensional statistics~\citep[e.g.,][]{bagnoli2006log, saumard2014log}.

The function $v^*(\vx)\pi(\vx)$ is well-rounded if for any $\cA \subseteq \cX$, the distribution defined by $f_{\pi}(\cA) = \frac{\int_{\cA} v^\ast(\vx) \pi(\vx) d\vx }{\int_{\cX} v^\ast(\vx) \pi(\vx) d\vx}$ is neither too spread out nor too concentrated. We include the formal definition in Appendix~\ref{app:population-model} (Def.~\ref{def:well_rounded}). Every log-concave function can be brought to a well-rounded position by an affine transformation of the
space in polynomial time~\citep{lovasz2006fast}.






\paragraph{Regret analysis.}

We now prove that the regret of Algorithm~\ref{alg:explore-then-commit} is $\tilde{O}(\sqrt{Td\ell})$. To do so, we must contend with two sources of error: the fact that we use the learned linear separators $\hat \vw$ instead of $\vw$ and the estimation error introduced by the sampling algorithm.

We begin our analysis by proving that for any $\vec{y} \in \{0,1\}^{\ell}$ in the image of $h_{\vw}: \cX \to \{0,1\}^{\ell}$, the agent's expected value conditioned on $\vx \in h_{\vwhat}^{-1}(\vec{y})$ is close to its true expected value conditioned on $\vx \in h_{\vw}^{-1}(\vec{y})$. In this section, we use the notation $\epsilon = \frac{\ell}{t'}\left(d\ln \frac{2et'}{d} + \ln \frac{2\ell}{\delta}\right).$

\begin{restatable}{lemma}{estValDiff}\label{lem:est-val-diff}
	For any $\vec{y} \in \{0,1\}^{\ell}$, with probability at least $1-\delta$ over $\vx_1, \dots, \vx_{t'} \sim \cP$, \begin{equation*}
	\left|\E\left[v^*(\vx) |\vx \in h_{\vwhat}^{-1}(\vec{y})\right] - \E\left[v^*(\vx)| \vx \in h_{\vw}^{-1}(\vec{y})\right]\right|\nonumber\leq H\epsilon.\end{equation*}
\end{restatable}


\begin{proof}
	We can decompose the set $h_{\vwhat}^{-1}(\vec{y})$ as \[h_{\vwhat}^{-1}(\vec{y}) = \left(h_{\vwhat}^{-1}(\vec{y}) \cap h_{\vw}^{-1}(\vec{y})\right) \cup  \left(h_{\vwhat}^{-1}(\vec{y}) \setminus h_{\vw}^{-1}(\vec{y})\right).\]
	We can therefore write \begin{align*}
		&\E_{\vx \sim \cP}\left[v^*(\vx) \mid \vx \in h_{\vwhat}^{-1}(\vec{y})\right]\\
		= &\E\left[v^*(\vx) \mid \vx \in \left(h_{\vwhat}^{-1}(\vec{y}) \cap h_{\vw}^{-1}(\vec{y})\right)\right]\\
		&\cdot \Pr\left[\vx\in \left(h_{\vwhat}^{-1}(\vec{y}) \cap h_{\vw}^{-1}(\vec{y})\right)\right]\\
		+&\E\left[v^*(\vx) \mid \vx \in \left(h_{\vwhat}^{-1}(\vec{y}) \setminus h_{\vw}^{-1}(\vec{y})\right)\right]\\
		&\cdot \Pr\left[\vx\in \left(h_{\vwhat}^{-1}(\vec{y}) \setminus h_{\vw}^{-1}(\vec{y})\right)\right].
	\end{align*}
	
	Similarly, we can write \begin{align*}&\E_{\vx \sim \cP}\left[v^*(\vx) \mid \vx \in h_{\vw}^{-1}(\vec{y})\right]\\
		= &\E\left[v^*(\vx) \mid \vx \in \left(h_{\vwhat}^{-1}(\vec{y}) \cap h_{\vw}^{-1}(\vec{y})\right)\right]\\
		&\cdot \Pr\left[\vx\in \left(h_{\vwhat}^{-1}(\vec{y}) \cap h_{\vw}^{-1}(\vec{y})\right)\right]\\
		+&\E\left[v^*(\vx) \mid \vx \in \left(h_{\vw}^{-1}(\vec{y}) \setminus h_{\vwhat}^{-1}(\vec{y})\right)\right]\\
		&\cdot \Pr\left[\vx\in \left(h_{\vw}^{-1}(\vec{y}) \setminus h_{\vwhat}^{-1}(\vec{y})\right)\right].\end{align*}
	
	Matching terms, we have that \begin{align*}&\left|\E_{\vx \sim \cP}\left[v^*(\vx) \mid \vx \in h_{\vwhat}^{-1}(\vec{y})\right] - \E\left[v^*(\vx) \mid \vx \in h_{\vw}^{-1}(\vec{y})\right]\right|\\
		=&\left| \E\left[v^*(\vx) \mid \vx \in \left(h_{\vwhat}^{-1}(\vec{y}) \setminus h_{\vw}^{-1}(\vec{y})\right)\right]\right.\\
		&\cdot \Pr\left[\vx\in \left(h_{\vwhat}^{-1}(\vec{y}) \setminus h_{\vw}^{-1}(\vec{y})\right)\right]\\
		&-\E\left[v^*(\vx) \mid \vx \in \left(h_{\vw}^{-1}(\vec{y}) \setminus h_{\vwhat}^{-1}(\vec{y})\right)\right]\\
		&\left.\cdot \Pr\left[\vx\in \left(h_{\vw}^{-1}(\vec{y}) \setminus h_{\vwhat}^{-1}(\vec{y})\right)\right]\right|.
	\end{align*}
	
	We know that $\vx \in \left(h_{\vwhat}^{-1}(\vec{y}) \setminus h_{\vw}^{-1}(\vec{y})\right)$ if and only if  $h_{\vwhat}(\vx) = \vec{y}$ and $h_{\vw}(\vx) \not= \vec{y}$, which means that $h_{\vwhat}(\vx) \not= h_{\vw}(\vx).$ The following claim bounds $\Pr[h_{\vwhat}(\vx) = \vec{y} \text{ and } h_{\vw}(\vx) \not= \vec{y}] \leq \Pr[h_{\vwhat}(\vx) \neq h_{\vw}(\vx)].$
	
	\begin{claim}\label{claim:offline-pac}
		With probability $1- \delta$, $\Pr_{\vx \sim \cP}[h_{\vwhat}(\vx) \neq h_{\vw}(\vx)] \leq \epsilon.$
\end{claim}

\begin{proof}[Proof of Claim~\ref{claim:offline-pac}]
	For a fixed $i \in [\ell]$, by the standard PAC learning generalization bound in the realizable setting (e.g., Theorem 4.8 by \citet{Anthony09:Neural}), we have that with probability $1- \frac{\delta}{\ell}$,
	\begin{align*}
		\Pr[\sgn(\vw_i \cdot \vec{x}) \not= \sgn (\vwhat_i \cdot \vec{x})] \leq \frac{1}{t'}\left(d\ln \frac{2et'}{d} + \ln \frac{2\ell}{\delta}\right).
	\end{align*}
	Therefore, with probability $1- \delta$,
	\begin{align*}
		&\Pr_{\vx \sim \cP}[h_{\vwhat}(\vx) \neq h_{\vw}(\vx)]\\
		= &\Pr[\exists i \in [\ell] \text{ such that } \sgn(\vw_i \cdot \vec{x}) \not= \sgn (\vwhat_i \cdot \vec{x})] \leq \epsilon,
	\end{align*}
	as claimed.
\end{proof}

By Claim~\ref{claim:offline-pac} and the fact that $v^*(\vec{x}) \in [0,H]$, we therefore know that $\E\left[v^*(\vx) \mid \vx \in \left(h_{\vwhat}^{-1}(\vec{y}) \setminus h_{\vw}^{-1}(\vec{y})\right)\right]\cdot \Pr_{\vx \sim \cP}\left[\vx \in \left(h_{\vwhat}^{-1}(\vec{y}) \setminus h_{\vw}^{-1}(\vec{y})\right)\right] \in [0,H\epsilon].$ By a symmetric argument, $\E\left[v^*(\vx) \mid \vx \in \left(h_{\vw}^{-1}(\vec{y}) \setminus h_{\vwhat}^{-1}(\vec{y})\right)\right]\cdot \Pr_{\vx \sim \cP}\left[\vx \in \left(h_{\vw}^{-1}(\vec{y}) \setminus h_{\vwhat}^{-1}(\vec{y})\right)\right] \in [0,H\epsilon].$
Therefore, the lemma statement holds.
\end{proof}

Lemma~\ref{lem:est-val-diff} guarantees that $\E_{\vx \sim \cP}\left[v^*(\vx) \mid \vx \in h_{\vwhat}^{-1}(\vec{y})\right]$ is a good approximation of $\E\left[v^*(\vx) \mid \vx \in h_{\vw}^{-1}(\vec{y})\right]$ ---which is the key quantity needed to compute the optimal policy (see Prop.~\ref{prop:opt}). However, this estimate may be difficult to compute when $\vx$ is high dimensional, despite the fact that $\cP$ is known. The integration algorithm of \citet{lovasz2006fast} allows us to estimate it in polynomial time, as we summarize in the following lemma.

\begin{restatable}{lemma}{insLossLogconcave}\label{lem:ins-loss-logconcave}
	Suppose that $v^\ast(\vx) \pi(\vx)$ is log-concave and well-rounded.
	Then for any $\vy \in \{0,1\}^{\ell}$, with probability at least $1 - \delta$,  we can compute a constant $A$ in polynomial time such that 
	$\left|A - \E_{\vx \sim \cP}\left[v^*(\vx) \mid \vx \in h_{\vw}^{-1}(\vy)\right]\right| \leq H\epsilon.$
\end{restatable}

\begin{proof}
	By Lemma~\ref{lem:est-val-diff}, 
	with probability at least $1-\nicefrac{\delta}{2}$, \begin{align*}&\left|\E_{\tilde{\vec{x}} \sim \cP}\left[v^*(\tilde{\vec{x}}) \mid \tilde{\vec{x}} \in h_{\vwhat}^{-1}(h_{\vw}(\vec{x}))\right]\right.\\
		&\left.- \E_{\tilde{\vec{x}} \sim \cP}\left[v^*(\tilde{\vec{x}}) \mid \tilde{\vec{x}} \in h_{\vw}^{-1}(h_{\vw}(\vec{x}))\right]\right| \leq \epsilon H.\end{align*}
	
	By definition
	\begin{align*}
		&\E_{\tilde{\vx} \sim \cP}[v^\ast(\tilde{\vx}) \mid\tilde{\vec{x}} \in h_{\vwhat}^{-1}(h_{\vw}(\vec{x}))]\\
		= &\int_{\tilde \vx \in h_{\vwhat}^{-1}(h_{\vw}(\vec{x}))} v^\ast(\tilde \vx)\pi(\tilde \vx) d\tilde\vx.
	\end{align*}
	
	Then, by~\citet[Theorem~1.3]{lovasz2006fast}, in polynomial runtime with probability of at least $1-\nicefrac{\delta}{2}$, we can compute a constant $A$, such that  
	\begin{align*}&\left|A -  \E_{\tilde{\vec{x}} \sim \cP}\left[v^*(\tilde{\vec{x}}) \mid \tilde{\vec{x}} \in h_{\vwhat}^{-1}(h_{\vw}(\vec{x}))\right] \right|\\
		\leq \text{ }&\epsilon \cdot \E_{\tilde{\vec{x}} \sim \cP}\left[v^*(\tilde{\vec{x}}) \mid \tilde{\vec{x}} \in h_{\vwhat}^{-1}(h_{\vw}(\vec{x}))\right] \leq \epsilon H.\end{align*}
	
	By the triangle inequality and a union bound, we have that with probability of at least $1 -\delta$, we can compute a value $A$ such that
	\[\left|A -  \E_{\tilde{\vec{x}} \sim \cP}\left[v^*(\tilde{\vec{x}}) \mid \tilde{\vec{x}} \in h_{\vw}^{-1}(h_{\vw}(\vec{x}))\right] \right| \leq 2\epsilon H,\]
	which completes the proof.
\end{proof}

\begin{restatable}{lemma}{poplossdiff}\label{lem:loss-diff}
	In each round $t \in \{t'+1, \dots, T\}$ of the exploitation phase in Algorithm~\ref{alg:explore-then-commit}, with probability at least $1-\delta$, the expected instantaneous regret incurred in round $t$ is at most \[\frac{2H\ell}{t'}\left(d\ln \frac{2et'}{d} + \ln \frac{\ell\cdot 2^{\ell+2}}{\delta}\right).\]
\end{restatable}

\begin{proof}
	Given $h_{\vw}(\vx_t)$ and price $p(\vx_t)$, denote the estimated value of $\E_{\vx \sim \cP}\left[v^*(\vx) \mid \vx \in h_{\vwhat}^{-1}(h_{\vw}(\vec{x}_t))\right]$ obtained using the sampling algorithm (Alg~\ref{alg:explore-then-commit}, step~\ref{alg:step-sampling}) as $A(h_{\vw}(\vx_t))$. For simplicity of notation, we denote the decision of the oracle policy as $s^\ast$ and the decision of the learned policy as $s_t$: 
	\begin{align*}
		s^\ast &=  \Ind\left(\E_{\vx \sim \cP}\left[v^*(\vx) \mid \vx \in h^{-1}_{\vec{w}}(h_{\vec{w}}(\vec{x}_t))\right] > p(\vec{x}_t)\right), \\
		s_t &=  \Ind\Big(A(h_{\vw}(\vx_t)) > p(\vec{x}_t)\Big).
	\end{align*}
	
	Now we bound the expected instantaneous regret in round $t$:
	\begin{align*}&\E_{\vx_t \sim \cP}\left[ \left(v^\ast(\vx_t) - p(\vx_t)\right)s^\ast - \left(v^\ast(\vx_t) - p(\vx_t)\right)s_t\right]\\
		= &\E_{\vx_t \sim \cP}\left[ \left(v^\ast(\vx_t) - p(\vx_t)\right)\left(s^\ast- s_t\right) \right].\end{align*}
	Let $\Delta$ denote the difference $\Delta = s^\ast- s_t$, so $\Delta \in \{-1, 0,1\}$ is a random variable that depends on $\vx_t$. By the law of total expectation,
	\begin{align*}
		&\E_{\vx_t \sim \cP}\left[ \left(v^\ast(\vx_t) - p(\vx_t)\right)s^\ast - \left(v^\ast(\vx_t) - p(\vx_t)\right)s_t\right]\\
		= &\E_{\vx_t \sim \cP}\left[\E_{\vx \sim \cP}\left[ \left( v^\ast(\vx) - p(\vx_t)\right)\Delta \mid \vx \in h_{\vw}^{-1} (h_{\vw}(\vx_t))\right] \right]\\
		= &\E_{\vx_t \sim \cP} \left[\left(\E_{\vx \sim \cP}\left[v^\ast(\vx)\vert \vx \in h_{\vw}^{-1} (h_{\vw}(\vx_t))\right] - p(\vx_t)\right)\Delta \right].\end{align*}
	The variable $\Delta$ is only nonzero when $s^*\not= s_t$. Let $E$ denote the event where $s^*\not= s_t$ and let $p_E = \prob_{\vx_t \sim \cP}[E]$. Then
	\begin{align*}
		&\E_{\vx_t \sim \cP}\left[ \left(v^\ast(\vx_t) - p(\vx_t)\right)s^\ast - \left(v^\ast(\vx_t) - p(\vx_t)\right)s_t\right]\\
		\leq &\E_{\vx_t} \left.\left[\left|\E_{\vx}\left[v^\ast(\vx)\mid \vx \in h_{\vw}^{-1} (h_{\vw}(\vx_t))\right] - p(\vx_t)\right| \, \right| \, E\right] \cdot p_E\\
		\leq &\E_{\vx_t} \left.\left[\left|\E_{\vx}\left[v^\ast(\vx)\mid \vx \in h_{\vw}^{-1} (h_{\vw}(\vx_t))\right] - p(\vx_t)\right| \, \right| \, E\right].
	\end{align*}
	
	By definition, when event $E$ happens, we know that
	\begin{align*}&A(h_{\vw}(\vx_t))< p(\vec{x}_t) \leq \E_{\vx}[v^*(\vx) \mid \vx \in h^{-1}_{\vec{w}}(h_{\vec{w}}(\vec{x}_t))] \text{ or }\\
		&\E_{\vx\sim \cP}[v^*(\vx) \mid \vx \in h^{-1}_{\vec{w}}(h_{\vec{w}}(\vx_t))]< p(\vx_t) \leq A(h_{\vw}(\vx_t)),
	\end{align*}
	where in either case we have that 
	\begin{align*}&\left|\E_{\vx \sim \cP}[v^*(\vx) \mid \vx \in h^{-1}_{\vec{w}}(h_{\vec{w}}(\vx_t))] - p(\vx_t) \right|\\
		\leq &\left|\E_{\vx \sim \cP}[v^*(\vx) \mid \vx \in h^{-1}_{\vec{w}}(h_{\vec{w}}(\vx_t))] - A(h_{\vw}(\vx_t)) \right|.
	\end{align*}
	Given $\delta' \in (0,1)$, let $\epsilon' = \frac{\ell}{t'}\left(d\ln \frac{2et'}{d} + \ln \frac{4\ell}{\delta'}\right)$. By Lemma~\ref{lem:ins-loss-logconcave}, for any value of $h_{\vw}(\vx) \in \cY$, with probability at least $1-\delta'$, 
	\begin{align*}
		\left|\E_{\vx \sim \cP}[v^*(\vx) \mid \vx \in h^{-1}_{\vec{w}}(h_{\vec{w}}(\vx_t))] - A(h_{\vw}(\vx_t))\right| \leq 2\epsilon' H.
	\end{align*}
	Setting $\delta' = \nicefrac{\delta}{|\cY|} = \nicefrac{\delta}{2^\ell}$, by a union bound over elements in $\cY$, we have that with probability at least $1-\delta$, 
	\begin{align*}
		&\E_{\vx_t \sim \cP}\left[ \left(v^\ast(\vx_t) - p(\vx_t)\right)s^\ast - \left(v^\ast(\vx_t) - p(\vx_t)\right)s_t\right]\\
		&\leq \E_{\vx_t \sim \cP} \left.\left[\left|\E_{\vx \sim \cP}\left[v^\ast(\vx)\mid \vx \in h_{\vw}^{-1} (h_{\vw}(\vx_t))\right] - p(\vx_t)\right| \, \right| \, E\right] \\
		&\leq \E_{\vx_t} \left.\left[\left|\E_{\vx}\left[v^\ast(\vx)\mid \vx \in h_{\vw}^{-1} (h_{\vw}(\vx_t))\right] - A(h_{\vw}(\vx_t))\right| \, \right| \, E\right]\\
		&\leq \frac{2H\ell}{t'}\left(d\ln \frac{2et'}{d} + \ln \frac{\ell\cdot 2^{\ell+2}}{\delta}\right),
	\end{align*}
	which completes the proof.
\end{proof}

This instantaneous regret bound then implies a regret bound for  Algorithm~\ref{alg:explore-then-commit}, the proof of which is in Appendix~\ref{app:population-model}.

\begin{restatable}{theorem}{popuRegret}\label{thm:popu-regret} With probability at least $1-\delta$, the regret of Algorithm~\ref{alg:explore-then-commit} is
$R_T = \cO(\sqrt{Td\ell\log (\nicefrac{T\ell}{\delta})}).$
\end{restatable}

In Algorithm~\ref{alg:explore-then-commit}, we assumed that we knew the horizon $T$. This assumption can be lifted via a doubling trick; see Appendix~\ref{app:population-model}.




\section{General Masking Functions}\label{sec:finite-adaptive}

We next consider a more general setting where in each round, an item $\vx_t$ is drawn from an unknown distribution $\cP$ and a published price $p_t$ is drawn from some fixed unknown distribution. We study the case where the masking function is an arbitrary mapping $h: \cX \to [n]$. In this setting, is there a no-regret strategy that the buyer can use?

We answer this question in the affirmative, building on the classical \expalg algorithm~\citep{beygelzimer2011contextual}. Out of the box, \expalg has per-round runtime that is exponential in $n$, but we exploit the structure of our problem setting to obtain a polynomial per-round runtime.
We prove that this algorithm has a regret bound of $\cO(\sqrt{T(n\log \nicefrac{T}{n} + \log \nicefrac{1}{\delta})})$ with probability at least $1-\delta$.

\begin{algorithm}[!t]
\caption{\expalg with an unknown distribution}\label{alg:exp4}
\begin{algorithmic}[1]
\STATE \textbf{Input:} $T \geq 0$, $\delta \in (0,1)$.
\STATE Set $\tau = \sqrt{Tn \log \frac{eT}{n} + \log \frac{2}{\delta} }$.
\\[1ex]
\FOR[\textcolor{blue}{Initialization phase}]{$t=1,2,\ldots, \tau$}
\STATE Receive $h(\vx_t)$, price $p_t$ where $\vx_t \stackrel{iid}{\sim} \cP$.
\STATE Make decision $b_t \sim Bern(0.5)$ at random.
\ENDFOR\\[1ex]
\FOR[\textcolor{blue}{Extract a finite subset of policies}]{$i=1,2,\ldots, n$}
\STATE Let $i_1, i_2, \dots, i_{m_i} \in [\tau]$ be the set of indices where $h(\vx_{i_j}) = i$ for all $j \in \{1, 2, \dots, m_i\}$
\STATE Define $\cV_i = \{0, p_{i_1}, p_{i_2}, \dots, p_{i_{m_i}}\}$
\ENDFOR
\STATE {Define $\cV = \times_{i = 1}^n \cV_i$}
\STATE{Set $\gamma = \sqrt{\frac{\log{|\cV|}}{2(T-\tau)}}$ and $w_{\tau+1}^{\vec{v}}=1$ for all $\vec{v} \in \cV$.}
\\[1ex]
\FOR[\textcolor{blue}{\texttt{Exp.4} subroutine}]{$t=\tau+1, \ldots, T$}
\STATE Receive $h(\vx_t)$, price $p_t$ where $\vx_t \stackrel{iid}{\sim} \cP$.
\STATE Get advice vectors $\vxi_t^{\vec{v}}\in \{0,1\}^2$ for all $\vec{v} \in \cV$ where $\vxi_t^{\vec{v}} = (1- \pi_{\vec{v}}(p_t, h(\vx_t)), \pi_{\vec{v}}(p_t, h(\vx_t)))$.
\STATE Set $W_t = \sum_{\vec{v} \in \cV} w_t^{\vec{v}}$ and define $\bar \vxi_t \in (0,1)^2$ as
\begin{align*}
    \bar \vxi_t[0]  &= (1-2\gamma)\sum_{\vec{v} \in \cV} \frac{w_t^{\vec{v}} \vxi_t^{\vec{v}}[0]}{W_t} + \gamma\\
    \bar \vxi_t[1]  &= (1-2\gamma)\sum_{\vec{v} \in \cV} \frac{w_t^{\vec{v}} \vxi_t^{\vec{v}}[1]}{W_t} + \gamma.
\end{align*}
\STATE Draw decision $b_t \sim Bern(\bar \vxi_{t}[1])$ and receive reward $u_t = (v^\ast(\vx_t) - p_t)b_t$.  \STATE Set $\hat \vr_t = (\nicefrac{b_t u_t}{\bar \vxi_{t}[1]},0)^\top$. \\[1ex]
\FOR{$\vec{v} \in \cV$}\label{alg:exp-loop}
\STATE{Set \[C_t^{\vec{v}} = \frac{\gamma}{2}\left(\vxi_t^{\vec{v}} \cdot \hat \vr_t + \sum_{b=0}^1\frac{ \vxi_t^{\vec{v}}[b]}{\bar \vxi_t^{\vec{v}}[b]}\sqrt{\frac{\log \nicefrac{|\cV|}{\delta}}{2(T-\tau)}}\right)\]}
\STATE{Set $w_{t+1}^{\vec{v}} = w_t^{\vec{v}} \exp C_t^{\vec{v}}$}
\ENDFOR\\[1ex]
\ENDFOR\\[1ex]
\end{algorithmic}
\end{algorithm}

\subsection{\expalg algorithm}

Based on the optimal strategy from Proposition~\ref{prop:opt}, we define an infinite set of policies that take as input $(p_t, h(\vx_t))$ and return decisions in $\{0,1\}$ indicating whether or not the buyer should buy the item. Each policy is defined by a vector $\vec{v} \in [0,H]^n$ as follows: $\pi_{\vec{v}}(p_t, h(\vx_t)) = \Ind(\vec{v}[h(\vx_t)] \geq p_t)$. The optimal strategy from Proposition~\ref{prop:opt} corresponds to the strategy $\pi_{\vec{v}}$ with $\vec{v} = (\E[v^*(\vx) \mid h(\vx) = 1], \dots, \E[v^*(\vx) \mid h(\vx) = n])$. We use the notation $\Pi = \{\pi_{\vec{v}} \mid \vec{v} \in [0,H]^n\}$ to denote the set of all such policies.


A key observation is that our problem can be framed as a contextual bandit problem with an oblivious adversary and an infinite set of contexts. At each round $t = 1, \dots, T$, the buyer observes stochastic context $(p_t, h(\vx_t))$ and makes a purchase decision. This is a contextual bandit problem with two arms: the first arm corresponds to the decision ``no purchase'' and the second arm corresponds to the decision ``purchase.'' The reward of the first arm is always zero, while the reward of pulling the second arm depends on the item $\vx_t$ and the price $p_t$.

We prove that the class of policies $\Pi$ has a VC dimension of $n$, which allows us to adapt \expalg from~\citet{beygelzimer2011contextual}, a generic contextual bandits algorithm for policy classes with finite VC dimension. Algorithm~\ref{alg:exp4} begins with an initialization phase of length $\tau$.
In this phase, the buyer chooses their action uniformly at random and collects tuples $\{(h(\vx_t), p_t)\}_{t=1}^{\tau}$. The algorithm then uses these tuples to identify a finite (though exponentially large), representative subset of policies in $\Pi$.
In particular, using the collected tuples, the algorithm partitions $\Pi$ into a \textit{finite} set of equivalence classes where two policies $\pi, \pi'$ are equivalent if they agree on the set of $\tau$ collected tuples. Then the buyer constructs a finite set of policies $\Pi'$ by selecting one policy from each equivalence class. Algorithm~\ref{alg:exp4} does this by first defining for each $i \in [n]$ a set $\cV_i \subset \R$ which is the set of all prices from the first $\tau$ rounds for items $\vx_t$ with $h(\vx_t) = i$. The finite set $\Pi'$ of policies is then defined as $\Pi' = \{\pi_{\vec{v}} : \vec{v} \in \times_{i = 1}^n \cV_i\}$. We prove that $\Pi'$ contains a policy from each equivalence class in Lemma~\ref{lem:exp4-pol-set}.

In the remaining rounds, Algorithm~\ref{alg:exp4} follows the \texttt{Exp4.P} strategy~\citep{beygelzimer2011contextual} which runs multiplicative weight updates on each of the selected policies $\pi_{\vec{v}}$ with $\vec{v} \in \times_{i = 1}^n \cV_i$.
Out-of-the-box, \expalg would therefore have a per-round runtime that is exponential in $n$ since $\times_{i = 1}^n \cV_i$ is exponentially large.
However, with a careful analysis, we show that in our setting these multiplicative weight updates can be computed in polynomial time.

\subsection{Regret}

The key first step is to show that although the set of all policies we need to consider
$\Pi$
is infinite, it has a \textit{finite} VC dimension. The full proof is in Appendix~\ref{app:proof-exp4}.

\begin{restatable}{lemma}{lemVC}\label{lem:exp4-vc}
The VC dimension of $\Pi$ is $n$.
\end{restatable}

\begin{proof}[Proof sketch]
First, we show that the functions in $\Pi$ cannot be used to label $n+1$ contexts in all possible ways. Given $n+1$ contexts $(h(\vx_1), p_1), \dots, (h(\vx_{n+1}), p_{n+1})$, by the pigeonhole principle there must exist at least two items $\vx_i$ and $\vx_j$ that have the same index: $h(\vx_i) = h(\vx_j)$. Therefore, for any policy $\pi_{\vec{v}}$, the decisions for these two items are determined by the same threshold $\vec{v}[h(\vx_i)] = \vec{v}[h(\vx_j)]$. Without loss of generality, assume that $p_i < p_j$. There is no policy $\pi_{\vec{v}}$ where the decision is to purchase item $j$ but not purchase item $i$ because this would imply that $p_j \leq \vec{v}[h(\vx_j)] = \vec{v}[h(\vx_i)] < p_i$. However, with fewer than $n+1$ items, since all items can use a different threshold, their decisions do not interfere with each other. 
\end{proof}

Next, in order to invoke the regret bound of \expalg, we verify that $\Pi'$ is a representative set of policies. Formally, suppose we partition $\Pi$ into a set of equivalence classes where policies $\pi$ and $\pi'$ are equivalent if they agree on the set of $\tau$ tuples collected in the initalization phase. We prove that $\Pi'$ contains a policy from each equivalence class.

\begin{restatable}{lemma}{lemVCPolicySet}\label{lem:exp4-pol-set}
Let $\cV$ be defined as in Algorithm~\ref{alg:exp4}.
Then
\begin{align*}
&\left\{\left(\pi_{\vec{v}}(h(\vx_1), p_1), \dots, \pi_{\vec{v}}(h(\vx_{\tau}), p_{\tau})\right) : \vec{v} \in [0,H]^n\right\}\\
=\,&\left\{\left(\pi_{\vec{v}}(h(\vx_1), p_1), \dots, \pi_{\vec{v}}(h(\vx_{\tau}), p_{\tau})\right) : \vec{v} \in \cV\right\}.
\end{align*}
\end{restatable}

The proof of this lemma can be found in Appendix~\ref{app:proof-exp4}.

Lemmas~\ref{lem:exp4-vc} and \ref{lem:exp4-pol-set} imply the following regret bound:\footnote{By running $n$ copies of \expalg in parallel for each context, we would obtain a regret bound of $\cO(\sqrt{Tn\log (\nicefrac{Tn}{\delta})})$, but by using a more careful analysis in this section, we improve the dependence on $n$.}

\begin{theorem}\label{thm:exp4}
With probability $1-\delta$, 
Algorithm~\ref{alg:exp4} achieves a regret rate that is
$R_T=\cO(\sqrt{T(n\log \nicefrac{T}{n} + \log \nicefrac{1}{\delta})}).$
\end{theorem}

\begin{proof}
This theorem follows directly from Lemmas~\ref{lem:exp4-vc} and \ref{lem:exp4-pol-set} and Theorem 5 of~\citet{beygelzimer2011contextual}.
\end{proof}

\subsection{Computational Complexity}

The key challenge in applying \expalg out-of-the-box is that it computes multiplicative weight updates over every policy $\pi_{\vec{v}}$ with $\vec{v} \in \cV$---an exponential number of policies.
We show that by exploiting our problem structure, we can perform these multiplicative weight updates in each round in polynomial time. In particular, we show that we can efficiently compute the purchase probabilities $\bar\vxi_t[0]$ and $\bar\vxi_t[1]$ without computing the multiplicative weights $w_t^{\vv}$ for each $\vv \in \cV$ explicitly, and therefore Algorithm~\ref{alg:exp4} can be run with polynomial per-round runtime.
Intuitively, rather than sum over every vector in $\cV = \times_{i = 1}^n \cV_i$ as in the definitions of $\bar\vxi_t[0]$ and $\bar\vxi_t[1]$, we show how to sum over individual elements in $\cup_{i = 1}^n \cV_i$, of which there are $\tau + n = \tilde\cO(\sqrt{Tn} + n).$
We provide the complete proof in Appendix~\ref{app:proof-exp4}.

\begin{restatable}{theorem}{expRunTime}\label{thm:exp-runtime}
The purchase probabilities $\bar\vxi_t[0]$ and $\bar\vxi_t[1]$ in Algorithm~\ref{alg:exp4} can be computed in $\cO(n + \tau) = \cO(n + \sqrt{Tn\log(\nicefrac{T}{n}) + \log(\nicefrac{1}{\delta})})$ time.
\end{restatable}

\begin{proof}[Proof sketch] Our proof begins with the observation that for each index $i \in [n]$, the thresholds $0 \leq p_{i_1} \leq \cdots \leq p_{i_{m_i}}$ in $\cV_i$ divide the price range $[0,H]$ into $m_i+1$ non-overlapping ``buckets'': $[0, p_{i_1}), [p_{i_2}, p_{i_3}),\dots, [p_{i_{m_i}}, H]$. Using the notation $m = \max_i m_i$, the total number of buckets is $\cO(mn)$. Each context $(h(\vx_t), p_t)$ corresponds to exactly one bucket: the bucket $[p_{i_j}, p_{i_{j+1}})$ containing $p_t$ where $h(\vx_t) = i$. Moreover, the decision of each policy in each bucket is constant since the policies all use the same thresholds, namely, the boundaries of these buckets. Given a vector $\vec{v} \in \cV$ and a bucket $k$, let $a^\vv_k \in \{0,1\}$ be the policy's recommendation of ``buy'' or ``do not buy'' for any item that falls in that bucket. This allows us to rewrite the policy decision $\vxi_t^{\vv}$ as an alternative sum: $\vxi_t^{\vv} = \left(\sum_{k: a^\vv_k=0}  \Ind(\text{item in }k), \sum_{k: a^\vv_k=1} \Ind(\text{item in }k) \right)^\top$. Intuitively, $\Ind(\text{item in }k)$ is only nonzero for the bucket that this item belongs to, and the policy's decision for the item is the same as the policy's decision for that bucket. Using this fine-grained argument, we then show that all the purchase probabilities $\bar\vxi_t[0]$ and $\bar\vxi_t[1]$ can be computed in polynomial time without explicitly computing the exponentially many weights $w_t^{\vec{v}}.$
\end{proof}

\begin{remark}\label{rem:sliv}
		Under \emph{adversarial} prices, we can run $n$ independent copies of an algorithm for Lipschitz contextual bandits (for example, the algorithm from Section 8.3 of the textbook by~\citet{slivkins2019introduction}) to obtain an expected regret bound of $\tilde O(T^{\nicefrac{2}{3}}n^{\nicefrac{1}{3}})$.
	\end{remark}



\section{Conclusion}\label{sec:conclu}

We presented learning algorithms for buyers who participate in auctions with limited item information. This model captures a broad set of practical applications, including advertising auctions. Our algorithms are no-regret with respect to an oracle buyer who has perfect knowledge of the distribution over items and the masking function that the seller uses to obfuscate the item information. We proposed no-regret learning algorithms in a variety of settings, including when the distribution over items is either known or unknown to the buyer, and when the prices are either stochastic or adversarial.

To the best of our knowledge, this is the first result on no-regret learning strategies for buyers with partial item information. Many interesting questions remain open for future research. First, we have assumed that the valuation is bounded, and it would be an interesting direction for future research to extend our results to unbounded distributions. Second, we focused on posted-prices in this work. It would be interesting to consider extensions to second-price auctions with partial item information. With stochastic prices, the problem seems potentially feasible, but adversarial prices might pose challenges because the adversary could block the learner from estimating the second-highest bid. Moreover, can the platform release partial information in a way that optimally trades off between revenue and privacy? When the distribution over items is unknown, our algorithm works with general masking functions. Can better purchasing strategies be developed by exploiting the properties of a specific set of masking functions?

\section*{Acknowledgements}
We wish to acknowledge support from the Vannevar Bush Faculty Fellowship program under grant number N00014-21-1-2941. WG acknowledges support from a Google PhD fellowship.

\bibliography{icml_ref}
\bibliographystyle{icml2022}

\newpage
\appendix
\onecolumn

\section{Proofs for Section~\ref{sec:prelim}}\label{app:prelim-proofs}

\opt*

\begin{proof}


For a decision $b\in \{0,1\}$, let $R_b$ denote the expected utility of buying or not buying an item given the published information and price:
\begin{align*}
    R_b  &\defeq \E[(v^\ast(\vx) - p(\vx))\cdot b \mid h(\vx), p(\vx)]\\
    &= \E[(v^\ast(\vx) - p(\vx) ) \cdot b \mid v^\ast(\vx) \geq p(\vx), h(\vx), p(\vx)] \cdot \prob[v^\ast(\vx)\geq p(\vx)\mid h(\vx), p(\vx)] \\
    &\quad - \E[(p(\vx) - v^\ast(\vx) ) \cdot b \mid v^\ast(\vx) < p(\vx), h(\vx), p(\vx)] \cdot \prob[v^\ast(\vx)< p(\vx)\mid h(\vx), p(\vx)].
\end{align*}
Then, by definition, the optimal strategy $s^\ast(h(\vx), p(\vx), p, \cP)  = \Ind(R_1 > R_0) = \Ind(R_1 > 0)$.

Further, letting $x^+ = \max\{0,x\}$, by the law of total expectation, 
\begin{align*}
    & \E[(v^\ast(\vx) - p(\vx))^+ \mid h(\vx), p(\vx)]\\
    &= \E[(v^\ast(\vx) - p(\vx)) \mid v^\ast(\vx) \geq p(\vx), h(\vx), p(\vx)] \cdot \prob[v^\ast(\vx) \geq p(\vx) \mid h(\vx), p(\vx)] \\
    &\quad+ \E[0\mid v^\ast(\vx) < p(\vx), h(\vx), p(\vx)] \cdot \prob(v^\ast(\vx) < p(\vx) \mid h(\vx), p(\vx)) \\
    &= \E[(v^\ast(\vx) - p(\vx)) \mid v^\ast(\vx) \geq p(\vx), h(\vx), p(\vx)] \cdot \prob[v^\ast(\vx) \geq p(\vx) \mid h(\vx), p(\vx)].
\end{align*}
Similarly, letting $x^- = -\min\{0, x\}$,
\begin{align*}&\E[(v^\ast(\vx) - p(\vx))^- \mid h(\vx), p(\vx)]\\
= \, &\E[(p(\vx) - v^\ast(\vx)) \mid v^\ast(\vx) < p(\vx), h(\vx), p(\vx)] \cdot \prob(v^\ast(\vx) < p(\vx) \mid h(\vx), p(\vx)).\end{align*}
Therefore, the optimal strategy is:
\begin{align*}s^\ast(h(\vx), p(\vx), h, p, \cP) &= \Ind\left(\E[(v^*(x) - p(\vx))^+ \mid h(\vx), p(\vx)] - \E[(v^*(x) - p(\vx))^- \mid h(\vx), p(\vx)] > 0\right)\\
&= \Ind\left(\E[(v^*(x) - p(\vx)) \mid h(\vx), p(\vx)] > 0\right) \\
&= \Ind\left(\E[v^*(x) \mid h(\vx), p(\vx)] > p(\vx)\right).
\end{align*} The lemma statement then follows from Equation~\eqref{eq:p}.
\end{proof}
\section{Proofs for Section~\ref{sec:population}} \label{app:population-model}

\begin{definition}[\citet{lovasz2006fast}]\label{def:well_rounded}
Define the centroid of $f_{\pi}$ as $c_{f} = \int \vx df_{\pi}$. Define the variance of $f_{\pi}$ as $\Var(f_{\pi}) = \int \|\vx-c_f\|^2 df_{\pi}$. Also denote $\cL(\theta)$ as the level set $\{\vx: v^\ast(\vx) \pi(\vx) \geq \theta\}$.
\end{definition}

\popuRegret*

\begin{proof}
	
	First, in the exploration phase, the total regret is upper bounded by $Ht'$.
	
	Now consider the regret in the exploitation phase. Notice that, we would have obtained $\{\vx_i, h_{\vw}(\vx_i)\}_{i=1}^{t'}$  number of i.i.d. samples from the exploration phase. 
	
	
	
	Let $\epsilon = \frac{\ell}{t'}\left(d\ln \frac{2et'}{d} + \ln \frac{\ell\cdot 2^{\ell+2}}{\delta'}\right)$. By Lemma~\ref{lem:loss-diff} and a union bound over all rounds, with probability at least $1-\delta' T$,  the total expected regret incurred by Algorithm~\ref{alg:explore-then-commit} is $R_T \leq Ht' + 2 \epsilon H(T-t').$
	
	Let $t' = \sqrt{Td\ell \log\left(\frac{2\ell}{\delta'}\right)} $, we have
	\[R_T \leq H \sqrt{Td\ell \log\left(\frac{2\ell}{\delta'}\right)} + 2H\ell d \sqrt{ \frac{T}{d\ell \log\left(\frac{2\ell}{\delta'}\right)}} \log\left(\frac{2\ell}{d}\sqrt{Td\ell \log\left(\frac{2\ell}{\delta'}\right)}\right) \\
	+2H\ell\sqrt{\frac{T}{d\ell \log\left(\frac{2\ell}{\delta'}\right)}} \log\left(\frac{\ell2^{\ell+2}}{\delta'}\right).\]
	
	Note that $\log\left(\frac{\ell2^{\ell+2}}{\delta'}\right)  = \log\left(\frac{\ell}{\delta'}\right) + \ell \log(4) \leq \log\left(\frac{\ell}{\delta'}\right) + d \log(4) $. Setting $\delta' = \nicefrac{\delta}{T}$, we have that with probability at least $1-\delta$,
	\begin{align}
		R_T = \tilde \cO\left(\sqrt{Td\ell\log\left(\frac{2\ell T}{\delta}\right)}\right).
	\end{align}
	This completes the proof.
\end{proof}

\subsection{Algorithm for an unknown horizon $T$}

In Theorem~\ref{thm:popu-regret}, we assumed that we knew the time horizon $T$, which allowed us to set the correct length for the exploration phase. This assumption can be lifted by using the doubling trick which runs the algorithm in independent intervals that are doubling in length, as summarized by Algorithm~\ref{alg:doubling}.
\begin{algorithm}[!ht]
	\caption{Explore-then-Commit with Unknown Horizon}\label{alg:doubling}
	\begin{algorithmic}[1]
	\STATE \textbf{Input:} starting epoch length $T_0$.
	\FOR{$i = 1, 2, \dots$}
	\STATE $T_i \gets 2^i T_0$.
	\STATE Run Algorithm~\ref{alg:explore-then-commit} with $T = T_i$.
	\ENDFOR
\end{algorithmic}
\end{algorithm}
The regret bound remains the same up to constant factors.


\begin{restatable}{corollary}{popuRegretDoubling}\label{thm:popu-regret-doubling} Suppose that $v^\ast(\vx) \pi(\vx)$ is logconcave and well-rounded. 
Then with probability at least $1-\delta$, the regret of Algorithm~\ref{alg:doubling} is
$R_T = \tilde \cO(\sqrt{Td\ell\log (\nicefrac{T\ell}{\delta})}).$
\end{restatable}

\begin{proof}
Denote the total expected accumulated regret as $R_T$, and the expected accumulated regret in each interval with length $T_i$ as $R_{T_i}$. Denote the number of intervals as $m \leq \log_2 \frac{2T}{T_0}$. Then, by Theorem~\ref{thm:popu-regret} we have that for some universal constant $C$ and with probability at least $1-\delta'm$:
\begin{align*}
    R_T &\leq \sum_{i=1}^m R_{T_i} \\
    &\leq \sum_{i=1}^m C \sqrt{T_i d\ell \log\left(\frac{2\ell T_i}{\delta'}\right)}\\
    &\leq C\sqrt{d\ell} \sum_{i=1}^m \sqrt{T_i\left(\log\left( \frac{2\ell}{\delta'}\right) + \log T_i\right)}\\
    &\leq  C\sqrt{d\ell} \left( \sqrt{\log\left(\frac{2\ell}{\delta'}\right)} \cdot \sum_{i=1}^m \sqrt{2^iT_0} + \sum_{i=1}^m \sqrt{2^i T_0 \log(2^i T_0)}\right)\\
    &= \cO\left( \sqrt{d\ell T\log\left(\frac{2\ell}{\delta'}\right)}\right).
\end{align*}
 Let $\delta' = \nicefrac{\delta}{m}$ and note that $m \leq \log_2 \frac{2T}{T_0}$, applying a union bound over all intervals, we have that with probability at least $1-\delta$,
 \[
 R_T = \tilde \cO\left( \sqrt{d\ell T\log\left(\frac{2\ell}{\delta}\right)}\right).
 \]
This completes the proof.  
\end{proof}



\section{Proofs for Section~\ref{sec:finite-adaptive}}\label{app:proof-exp4}

\lemVC*

\begin{proof}

We argue that any function contained in $\Pi$ cannot be used to label $n+1$ input points in all possible ways. For simplicity denote $h(\vx) = y \in [n]$. Consider a set of input points $\{y_i, p_i\}_{i=1}^{n+1}$.
By the pigeonhole principle there must exist at least two elements $(y_i, p_i), (y_j, p_j)$ such that $y_i = y_j$ and $p_i \neq p_j$. Therefore by definition, we have that for any $\vv \in \R^n$,
\begin{align*}
    \pi_{\vv}(y_i, p_i) &= \Ind(\vv[y_i] > p_i) = \Ind(\vv[y_j] > p_i), \\
    \pi_{\vv}(y_j, p_j) &= \Ind(\vv[y_j] > p_j) = \Ind(\vv[y_i] > p_j).
\end{align*}
Without loss of generality, assume that $p_i < p_j$. Then the pair of labels $\pi_{\vv}(y_i, p_i) = 1$ and $\pi_{\vv}(y_j, p_j) = 0$ can never be achieved for any $\vv \in \R^n$. Thus $VCdim(\Pi) < n+1$.

Next, consider a set of $n$ input points $\{(i, 0.5)\}_{i=1}^{n}$. Each point in this set is labeled using a different index $i$, so all possible combinations of labels can be achieved using vectors $\vec{v} \in \{0,1\}^n$.  Thus we conclude that $VCdim(\Pi) = n$.

\end{proof}

\textbf{Lemma~\ref{lem:exp4-pol-set}.} \textit{Let $\{(h(\vx_1), p_1), \dots, (h(\vx_{\tau}), p_{\tau})\}$  be a subset of $[n] \times [0,H]$. For each $i \in [n]$, let $i_1, i_2, \dots, i_{m_i} \in [\tau]$ be the set of indices where $h(\vx_{i_j}) = i$ for all $j \in \{1, 2, \dots, m_i\}$. Define $\cV_i = \{0, p_{i_1}, p_{i_2}, \dots, p_{i_{m_i}}\}$ and $\cV = \times_{i = 1}^n \cV_i$. Then
\begin{align*}
\left\{\begin{pmatrix}\pi_{\vec{v}}(h(\vx_1), p_1)\\\vdots\\\pi_{\vec{v}}(h(\vx_{\tau}), p_{\tau})\end{pmatrix}\right\}_{\vec{v} \in [0,H]^n}= \quad \left\{\begin{pmatrix}\pi_{\vec{v}}(h(\vx_1), p_1)\\\vdots\\\pi_{\vec{v}}(h(\vx_{\tau}), p_{\tau})\end{pmatrix}\right\}_{\vec{v} \in \cV}.
\end{align*}}

\begin{proof}
We will show that for every $\vec{v} \in [0,H]^n$, there exists a vector $\vec{v}_0 \in \times_{i = 1}^n \cV_i$ such that $\pi_{\vec{v}}(h(\vx_j), p_j) = \pi_{\vec{v}_0}(h(\vx_j), p_j)$ for every $j \in [\tau]$. To this end, fix an index $i \in [n]$ and without loss of generality, let $i_1, i_2, \dots, i_{m_i}$ be sorted such that $0 := p_{i_0} < p_{i_1} <  p_{i_2} < \cdots < p_{i_{m_i}}$. Let $i' \in \{0,1, 2, \dots, i_{m_i}\}$ be the largest index such that $\vec{v}[i] \geq p_{i'}$. Define $\vec{v}_0[i] = p_{i'}.$ For every index $i_j \leq i'$, we know that $\vec{v}[i] \geq p_{i'} = \vec{v}_0[i] \geq p_{i_j}$, so $\pi_{\vec{v}}(h(\vx_{i_j}), p_{i_j})  = \Ind(\vec{v}[i] \geq p_{i_j}) = 1 =  \Ind(\vec{v}_0[i] \geq p_{i_j}) = \pi_{\vec{v}_0}(h(\vx_{i_j}), p_{i_j})$. Meanwhile, for every index $i_j > i'$, we know that $p_{i'} = \vec{v}_0[i] \leq \vec{v}[i] <p_{i_j}$. Therefore, $\pi_{\vec{v}}(h(\vx_{i_j}), p_{i_j}) = 0 = \pi_{\vec{v}_0}(h(\vx_{i_j}), p_{i_j})$. In either case, we have that $\pi_{\vec{v}}(h(\vx_{i_j}), p_{i_j}) = \pi_{\vec{v}_0}(h(\vx_{i_j}), p_{i_j})$. Since this is true for every index $i \in [n]$, the lemma statement holds.
\end{proof}

\expRunTime*

\begin{proof}
Label the elements of $\cV_i$ as $0 := p_{i,0} < p_{i,1} < \cdots < p_{i,m_i}$. Let $m = \max m_i$. For the ease of notation, define the variables $p_{i,m_i + 1} = p_{i, m_i +2} = \cdots p_{i,m} = H$. For each $i \in [n]$, $j \in \{1, \dots, m\}$, and $t \in \{\tau+1, \tau_2, \dots , T\}$, we define the variable \[d_{i,j}(t) = \Ind(h(\vx_t) = i \text{ and }p_t \in (p_{i, j-1}, p_{i,j}]).\] We also set $d_{i,0}(t) = \Ind(h(\vx_t) = i \text{ and }p_t =0).$ Next, for each $\vec{v} \in \times_{i = 1}^n \cV_i$, $i \in [n]$, and $j \in \{0, \dots, m\}$, we define the variable $a_{i,j}^{\vec{v}} = \Ind(\vec{v}[i] \geq p_{i,j})$.

\begin{claim}\label{claim:sum}
For $b \in \{0,1\}$, \begin{equation}\vxi^{\vec{v}}_t[b] = \sum_{(i,j) : a_{i,j}^{\vec{v}} = b} d_{i,j}(t).\label{eq:sum}\end{equation}
\end{claim}

\begin{proof}[Proof of Claim~\ref{claim:sum}]
First, let $i_t = h(\vx_t)$. If $p_t = 0$, define $j_t =0$ and otherwise, define $j_t$ such that $p_t \in (p_{i_t, j_t - 1}, p_{i_t,j_t}]$. Therefore, $d_{i,j}(t) = 1$ if $(i,j) = (i_t, j_t)$ and $d_{i,j}(t) = 0$ otherwise. This means that \begin{equation}\sum_{(i,j) : a_{i,j}^{\vec{v}} = b} d_{i,j}(t) = \Ind(a_{i_t,j_t}^{\vec{v}} = b).\label{eq:simplify}\end{equation}

We split the proof into two cases: $b = 0$ and $b = 1$.

\paragraph{Case 1: $b = 0$.}
We know that $\vxi^{\vec{v}}_t[0] = 1- \pi_{\vec{v}}(p_t, h(\vx_t)) = \Ind(\vec{v}[h(\vx_t)] < p_t) = \Ind(\vec{v}[i_t] < p_t).$

If $\vxi^{\vec{v}}_t[0] =0$, then $\vec{v}[i_t] \geq p_t$. Since $\vec{v} \in \cV$, we know that $\vec{v}[i_t] = p_{i_t,j}$ for some $j \in [m]$. Moreover, since $p_t \in (p_{i_t, j_t - 1}, p_{i_t,j_t}]$, the fact that $\vec{v}[i_t] \geq p_t$ means that $\vec{v}[i_t] \geq p_{i_t,j_t}$. Therefore, $a_{i_t, j_t}^{\vec{v}} = 1$. By Equation~\eqref{eq:simplify}, this means that $\sum_{(i,j) : a_{i,j}^{\vec{v}} = 0} d_{i,j}(t) = 0$, so Equation~\eqref{eq:sum} holds.

Meanwhile, if $\vxi^{\vec{v}}_t[0] =1$, then $\vec{v}[i_t] < p_t \leq p_{i_t,j_t}$. Therefore, $a_{i_t, j_t}^{\vec{v}} = 0$. By Equation~\eqref{eq:simplify}, this means that $\sum_{(i,j) : a_{i,j}^{\vec{v}} = 0} d_{i,j}(t) = 1$, so Equation~\eqref{eq:sum} holds.

\paragraph{Case 2: $b = 1$.}
We know that $\vxi^{\vec{v}}_t[1] = \Ind(\vec{v}[i_t] \geq p_t).$

If $\vxi^{\vec{v}}_t[1] =0$, then $\vec{v}[i_t] < p_t$. By the same logic as the previous case, this means that $a_{i_t, j_t}^{\vec{v}} = 0$. By Equation~\eqref{eq:simplify}, this means that $\sum_{(i,j) : a_{i,j}^{\vec{v}} = 1} d_{i,j}(t) = 0$, so Equation~\eqref{eq:sum} holds.

Meanwhile, if $\vxi^{\vec{v}}_t[1] =1$, then $\vec{v}[i_t] \geq p_t$. By the same logic as the previous case, $a_{i_t, j_t}^{\vec{v}} = 1$. By Equation~\eqref{eq:simplify}, this means that $\sum_{(i,j) : a_{i,j}^{\vec{v}} = 1} d_{i,j}(t) = 1$, so Equation~\eqref{eq:sum} holds.
\end{proof}
This means that \begin{align*}\vxi_t^{\vec{v}} \cdot \hat \vr_t &= \sum_{b = 0}^1 \sum_{(i,j) : a_{i,j}^{\vec{v}} = b} d_{i,j}(t) \hat\vr_t[b]\\
&= \sum_{i = 1}^n \sum_{j = 0}^m d_{i,j}(t) \left(\hat\vr_t[0]\Ind(a_{i,j}^{\vec{v}} = 0) + \hat\vr_t[1]\Ind(a_{i,j}^{\vec{v}} = 1)\right)\\
=  &\sum_{i = 1}^n \sum_{j = 0}^m d_{i,j}(t) \hat\vr_t[a_{i,j}^{\vec{v}}].\end{align*}

Similarly, \begin{align*}\sum_{b = 0}^1 \frac{\vxi_t^{\vec{v}}[b]}{\bar{\vxi}_t^{\vec{v}}[b]} = \sum_{b=0}^1 \sum_{(i,j) : a_{i,j}^{\vec{v}} = b} \frac{d_{i,j}(t)}{\bar{\vxi}_t^{\vec{v}}[b]} = \sum_{i =1}^n \sum_{j = 0}^m \frac{d_{i,j}(t)}{\bar{\vxi}_t^{\vec{v}}[a_{i,j}^{\vec{v}}]}.\end{align*}

Therefore, \[w_{t+1}^{\vec{v}} = w_t^{\vec{v}}\exp\left(\frac{\gamma}{2}\left(\vxi_t^{\vec{v}} \cdot \hat \vr_t + \sum_{b=0}^1\frac{ \vxi_t^{\vec{v}}[b]}{\bar \vxi_t^{\vec{v}}[b]}\sqrt{\frac{\log \nicefrac{|\cV|}{\delta}}{2(T-\tau)}}\right)\right)= w_t^{\vec{v}}\exp\left(\sum_{i =1}^n \sum_{j = 0}^m d_{i,j}(t) f_{a_{i,j}^{\vec{v}}}(t)\right),
\] where \[f_{a_{i,j}^{\vec{v}}}(t) = \frac{\gamma}{2}\left(\hat\vr_t[a_{i,j}^{\vec{v}}] + \frac{1}{\bar{\vxi}_t^{\vec{v}}[a_{i,j}^{\vec{v}}]}\sqrt{\frac{\log \nicefrac{|\cV|}{\delta}}{2(T-\tau)}}\right).\]

We can therefore write \begin{align*}w_{t+1}^{\vec{v}} &= \prod_{\tau=1}^t \exp\left(\sum_{i =1}^n \sum_{j = 0}^m d_{i,j}(\tau) f_{a_{i,j}^{\vec{v}}}(\tau)\right)\\
&= \exp\left(\sum_{\tau=1}^t\sum_{i =1}^n \sum_{j = 0}^m d_{i,j}(\tau) f_{a_{i,j}^{\vec{v}}}(\tau)\right)\\
&= \exp\left(\sum_{i =1}^n \sum_{j = 0}^m \sum_{\tau=1}^t d_{i,j}(\tau) f_{a_{i,j}^{\vec{v}}}(\tau)\right)\\
&= \prod_{i=1}^n\prod_{j=0}^m\exp\left(\sum_{\tau=1}^t  d_{i,j}(\tau) f_{a_{i,j}^{\vec{v}}}(\tau)\right).
\end{align*}
Letting $g_{i,j}(t,b) = \exp\left(\sum_{\tau=1}^t  d_{i,j}(\tau) f_{b}(\tau)\right)$, we have that \[w_{t+1}^{\vec{v}} = \prod_{i=1}^n\prod_{j=0}^m g_{i,j}(t, a_{i,j}^{\vec{v}}).\]

Let $\vec{b}_1, \dots, \vec{b}_m$ be the set of $m$ increasing bit vectors $\vec{b}_1 = (1,0,0,\dots, 0)$, $\vec{b}_2 = (1, 1, 0, \dots, 0)$, $\vec{b}_3 = (1, 1, 1, \dots, 0)$, \dots, $\vec{b}_m = (1, 1, 1,\dots, 1)$.
For each $i \in [n]$ and $\vec{b}_j$, define \[g_i(t, \vec{b}_j) = g_{i,1}(t, b_j[1])g_{i,2}(t, b_j[2])\cdots g_{i,m}(t, b_j[m]).\] We now prove the following claim.
\begin{claim}\label{claim:normalizing}
The normalizing constant $W_t$ can be computed in polynomial time as \[W_t = \prod_{i = 1}^n \sum_{j = 1}^{m_i} g_i(t, \vec{b}_j).\]
\end{claim}
\begin{proof}[Proof of Claim~\ref{claim:normalizing}] We first write
\begin{align*}W_t &= \sum_{\vec{v} \in \cV} w_t^{\vec{v}}\\
&= \sum_{\vec{v} \in \cV} \prod_{i=1}^n\prod_{j=0}^m g_{i,j}(t, a_{i,j}^{\vec{v}})\\
&= \sum_{\vec{v} \in \cV} \prod_{i=1}^n\prod_{j=0}^m g_{i,j}(t, \Ind(\vec{v}[i] \geq p_{i,j}))\\
&= \sum_{j_1=0}^{m_1}\dots \sum_{j_n = 0}^{m_n} \prod_{i=1}^n\prod_{j=0}^m g_{i,j}(t, \Ind(p_{i,j_i} \geq p_{i,j})).\end{align*}
This last equality holds because $\cV = \times_{i = 1}^n \cV_i$ and $\cV_i = \{p_{i,0}, \dots, p_{i, m_i}\}$. Moreover, \[W_t = \sum_{j_1=0}^{m_1}\dots \sum_{j_n = 0}^{m_n} \prod_{i=1}^n\prod_{j=0}^{j_i} g_{i,j}(t, 1)\prod_{j=j_i+1}^{m} g_{i,j}(t, 0)\] because $p_{i,j_i} \geq p_{i,j}$ for all $j \leq j_i$ and $p_{i,j_i} < p_{i,j}$ for all $j > j_i$.

By definition of the vectors $\vec{b}_1, \dots, \vec{b}_m$, \[W_t  =\sum_{j_1=0}^{m_1}\dots \sum_{j_n = 0}^{m_n} \prod_{i=1}^ng_i(t, \vec{b}_{j_i}) = \prod_{i = 1}^n \sum_{j = 1}^{m_i} g_i(t, \vec{b}_j),\] as claimed.
\end{proof}

We next prove that $\bar \vxi_t$ can be computed in polynomial time.
To do so, let $h(\vx_t) = i_t$. If $p_t = 0$, define $j_t =0$ and otherwise, define $j_t$ such that $p_t \in (p_{i_t, j_t - 1}, p_{i_t,j_t}]$.
Next, define $\bar{m}_i$ as follows:
\[\bar{m}_i = \begin{cases} j_t - 1 & \text{if } i = i_t\\
m_i &\text{otherwise.}\end{cases}\] Similarly, define $\underline{m}_i$ as follows: \[\underline{m}_i = \begin{cases} j_t & \text{if } i = i_t\\
0 &\text{otherwise.}\end{cases}\] Then $\bar \vxi_t$ has the following form:

\begin{claim}\label{claim:probabilities}
The probabilities $\bar \vxi_t$ can be computed in polynomial time as:
\[\bar \vxi_t[0] = \frac{1}{W_t}\prod_{i=1}^n \sum_{j_i=0}^{\bar{m}_i} g_i(t, \vec{b}_{j_i})\] and \[\bar \vxi_t[1] = \frac{1}{W_t}\prod_{i=1}^n \sum_{j_i=\underline{m}_i}^{m_i} g_i(t, \vec{b}_{j_i}).\]
\end{claim}

\begin{proof}[Proof of Claim~\ref{claim:probabilities}]
Recall that $\vxi_t^{\vec{v}}[0] = \Ind(\vec{v}[i_t] < p_t).$ Therefore,
\begin{align*}\sum_{\vec{v} \in \cV} w_t^{\vec{v}} \vxi_t^{\vec{v}}[0] &= \sum_{\vec{v} \in \cV} \prod_{i=1}^n\prod_{j=0}^m g_{i,j}(t, a_{i,j}^{\vec{v}})\vxi_t^{\vec{v}}[0]\\
&= \sum_{\vec{v} \in \cV} \prod_{i=1}^n\prod_{j=0}^m g_{i,j}(t, a_{i,j}^{\vec{v}})\Ind(\vec{v}[i_t] < p_t)\\
&= \sum_{j_1=0}^{m_1}\dots \sum_{j_n = 0}^{m_n} \left(\prod_{i=1}^n g_i(t, \vec{b}_{j_i})\right)\Ind(p_{i_t,j_{i_t}} < p_t).
\end{align*}
This last equality holds because $\cV = \times_{i = 1}^n \cV_i$ and $\cV_i = \{p_{i,0}, \dots, p_{i, m_i}\}$. Without loss of generality, suppose that $i_t = 1$, so $\Ind(p_{i_t,j_{i_t}} < p_t) = \Ind(p_{1,j_1} < p_t)$. Then \begin{align*}\sum_{\vec{v} \in \cV} w_t^{\vec{v}} \vxi_t^{\vec{v}}[0] &= \sum_{j_1=0}^{m_1}\dots \sum_{j_n = 0}^{m_n} \left(\prod_{i=1}^n g_i(t, \vec{b}_{j_i})\right)\Ind(p_{1,j_1} < p_t)\\
&= \sum_{j_1=0}^{m_1}\Ind(p_{1,j_1} < p_t) \left(\sum_{j_2 = 0}^{m_2}\dots \sum_{j_n = 0}^{m_n} \left(\prod_{i=1}^n g_i(t, \vec{b}_{j_i})\right)\right).
\end{align*}
Since $p_t = 0$ if $j_t =0$ and otherwise $p_t \in (p_{1, j_t - 1}, p_{1,j_t}]$ we have that $\Ind(p_{1,j_1} < p_t) = 1$ if and only if $j_1 \leq j_t - 1$. Therefore, \[\sum_{\vec{v} \in \cV} w_t^{\vec{v}} \vxi_t^{\vec{v}}[0] = \sum_{j_1=0}^{j_t-1} \left(\sum_{j_2 = 0}^{m_2}\dots \sum_{j_n = 0}^{m_n} \left(\prod_{i=1}^n g_i(t, \vec{b}_{j_i})\right)\right) = \sum_{j_1=0}^{\bar{m}_1} \dots \sum_{j_n = 0}^{\bar{m}_n} \prod_{i=1}^n g_i(t, \vec{b}_{j_i}) = \prod_{i=1}^n \sum_{j_i=0}^{\bar{m}_i} g_i(t, \vec{b}_{j_i}).
\]
Similarly, since $\vxi_t^{\vec{v}}[1] = \Ind(\vec{v}[i_t] \geq p_t)$, we have
\[\sum_{\vec{v} \in \cV} w_t^{\vec{v}} \vxi_t^{\vec{v}}[1] = \sum_{j_1=0}^{m_1}\dots \sum_{j_n = 0}^{m_n} \left(\prod_{i=1}^n g_i(t, \vec{b}_{j_i})\right)\Ind(p_{i_t,j_{i_t}} \geq p_t).
\]
Without loss of generality, suppose that $i_t = 1$, so $\Ind(p_{i_t,j_{i_t}} \geq p_t) = \Ind(p_{1,j_1} \geq p_t)$. Then \[\sum_{\vec{v} \in \cV} w_t^{\vec{v}} \vxi_t^{\vec{v}}[1] = \sum_{j_1=0}^{m_1}\Ind(p_{1,j_1} \geq p_t) \left(\sum_{j_2 = 0}^{m_2}\dots \sum_{j_n = 0}^{m_n} \left(\prod_{i=1}^n g_i(t, \vec{b}_{j_i})\right)\right).
\]
Since $p_t = 0$ if $j_t =0$ and otherwise $p_t \in (p_{1, j_t - 1}, p_{1,j_t}]$ we have that $\Ind(p_{1,j_1} \geq p_t) = 1$ if and only if $j_1 \geq j_t$. Therefore, \begin{align*}\sum_{\vec{v} \in \cV} w_t^{\vec{v}} \vxi_t^{\vec{v}}[0] &= \sum_{j_1=j_t}^{m_1} \left(\sum_{j_2 = 0}^{m_2}\dots \sum_{j_n = 0}^{m_n} \left(\prod_{i=1}^n g_i(t, \vec{b}_{j_i})\right)\right)\\
&= \sum_{j_1=\underline{m}_1}^{m_1} \dots \sum_{j_n = \underline{m}_n}^{m_n} \prod_{i=1}^n g_i(t, \vec{b}_{j_i})\\
&= \prod_{i=1}^n \sum_{j_i=\underline{m}_i}^{m_i} g_i(t, \vec{b}_{j_i}).
\end{align*}
\end{proof}

Lastly, note that in the initialization phase of Algorithm~\ref{alg:exp4}, every  decision $b_t$ is computed with $\cO(1)$ time. In the \texttt{Exp4} subroutine, by Claim~\ref{claim:normalizing} and Claim~\ref{claim:probabilities}, every iteration is also computed in polynomial time. Therefore the theorem statement holds.
\end{proof}



\end{document}